\documentclass{article}

\usepackage{microtype}
\usepackage{graphicx}
\usepackage{booktabs} 

\usepackage{amsmath}
\usepackage{amsthm}
\usepackage{relsize}
\usepackage{nicefrac}
\usepackage{url}
\usepackage{sistyle}
\usepackage{subfig}

\usepackage{hyperref}

\usepackage[accepted]{icml2018}

\icmltitlerunning{Learning the Reward Function for a Misspecified Model}

\SIthousandsep{,}

\newcommand{\citenoun}[1]{\citet{#1}}

\DeclareMathOperator*{\expect}{{\mathlarger {\mathbf E}}}

\newcommand{\expects}{\expect\nolimits}

\newtheorem{lemma}{Lemma}
\newtheorem{theorem}[lemma]{Theorem}

\renewcommand{\algorithmiccomment}[1]{$\triangleright$ #1}

\begin{document}

\twocolumn[
\icmltitle{Learning the Reward Function for a Misspecified Model}




\begin{icmlauthorlist}
\icmlauthor{Erik Talvitie}{fandm}
\end{icmlauthorlist}

\icmlaffiliation{fandm}{Department of Computer Science, Franklin \&
  Marshall College, Lancaster, Pennsylvania, USA}

\icmlcorrespondingauthor{Erik Talvitie}{erik.talvitie@fandm.edu}

\icmlkeywords{model-based reinforcement learning; Markov decision processes}

\vskip 0.3in
]



\printAffiliationsAndNotice{}  

\begin{abstract}
  In model-based reinforcement learning it is typical to decouple the
  problems of learning the dynamics model and learning the reward
  function. However, when the dynamics model is flawed, it may
  generate erroneous states that would never occur in the true
  environment. It is not clear {\em a priori} what value the reward
  function should assign to such states. This paper presents a novel
  error bound that accounts for the reward model's behavior in states
  sampled from the model. This bound is used to extend the existing
  Hallucinated DAgger-MC algorithm, which offers theoretical
  performance guarantees in deterministic MDPs that do not assume a
  perfect model can be learned. Empirically, this approach to reward
  learning can yield dramatic improvements in control performance when
  the dynamics model is flawed.
\end{abstract}

\section{Introduction} \label{sec:intro}

In the reinforcement learning problem, an agent interacts with an
environment, receiving rewards along the way that indicate the quality
of its decisions. The agent's task is to learn to behave in a way that
maximizes reward. Model-based reinforcement learning (MBRL) techniques
approach this problem by learning a predictive model of the
environment and applying a planning algorithm to the model to make
decisions. Intuitively and theoretically \cite{szita2010model}, there
are many advantages to learning a model of the environment, but MBRL is
challenging in practice, since even seemingly minor flaws in the model
or the planner can result in catastrophic failure. As a result,
model-based methods have generally not been successful in large-scale
problems, with only a few notable exceptions
\citep[e.g.][]{abbeel2007application}.

This paper addresses an important but understudied problem in MBRL:
learning a reward function. It is common for work in model learning to
ignore the reward function \citep[e.g.][]{bellemare2014skip,
  oh2015action, chiappa2017recurrent} or, if the model will be used
for planning, to assume the reward function is given
\citep[e.g.][]{ross2012agnostic, talvitie2017self,
  ebert2017self}. Indeed, it is true that if an accurate model of the
environment's dynamics can be learned, reward learning is relatively
straightforward -- the two problems can be productively
decoupled. However, in this paper we will see that when the model
class is {\em misspecified} (i.e. that the representation does not
admit a perfectly accurate model), as is inevitable in problems of
genuine interest, the two learning problems are inherently entangled.

\subsection{An Example}\label{sec:example}

To better understand how the limitations of the dynamics model impact
reward learning, consider Shooter, a simplified video game example
introduced by \citet{talvitie2015agnostic}, pictured in Figure \ref{fig:shooter}. At
the bottom of the screen is a spaceship which can move left and right
and fire bullets, which fly upward. When the ship fires a bullet the
agent receives -1 reward. Near the top of the screen are three
targets. When a bullet hits a target in the middle (bullseye), the
target explodes and the agent receives 20 reward; otherwise a hit is
worth 10 reward. Figure \ref{fig:shooter} shows the explosions that
indicate how much reward the agent receives.

\begin{figure}
  \centering
  \includegraphics[width=.9\linewidth]{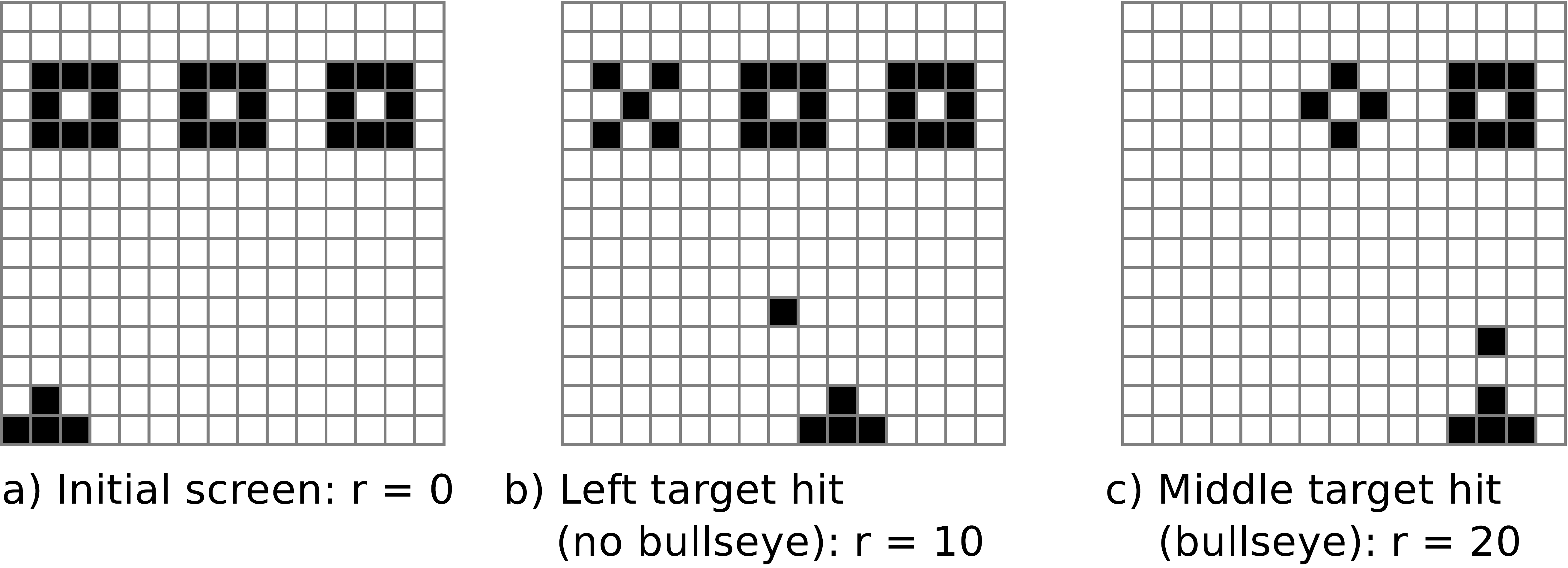}
  \caption{The Shooter domain.}
  \label{fig:shooter}
\end{figure}

It is typical to decompose the model learning problem into two
objectives: dynamics learning and reward learning. In the former the
agent must learn to map an input state and action to the next
state. In the latter the agent must learn to map a state and action to
reward. In this example the agent might learn to associate the
presence of explosions with reward. However, this decomposed approach
can fail when the dynamics model is imperfect.

For instance, say the dynamics model in this case is a factored MDP,
which predicts the value of each pixel in the next image based on the
$7 \times 5$ neighborhood centered on the pixel. Figure
\ref{fig:shooterRollout}b shows a short {\em sample rollout} from such
a model, sampling each state based on the previous sampled state. The
second image in the rollout illustrates the model's flaw: when
predicting the pixel marked with a question mark the model cannot
account for the presence of the bullet under the target. Hence, errors
appear in the subsequent image (marked with red outlines).

What reward should be associated with this erroneous image? The value
the learned model assigns will have a dramatic impact on the extent to
which the model is useful for planning and yet it is clear that no
amount of traditional data associating environment states with rewards
can answer this question. Even a provided, ``perfect'' reward function
would not answer this question; a reward function could assign any
value to this state and still be perfectly accurate in states that are
reachable in the environment. Intuitively it seems that the best case
for planning would be to predict 20 reward for the flawed state,
preserving the semantics that a target has been hit in the
bullseye. Note, however that this interpretation of the image is
specific to this particular flawed model; the reward model's quality
depends on its behavior in states generated by the {\em model} rather than
the environment.

The remainder of this paper formalizes this intuition. Section
\ref{sec:rwderr} presents a novel error bound on value functions in
terms of reward error, taking into account the rewards in flawed states
generated by the model. In Section \ref{sec:implications} the practical
implications of this theoretical insight are discussed, leading to an
extension of the existing Hallucinated DAgger-MC algorithm, which
provides theoretical guarantees in deterministic MDPs, even when the
model class is misspecified. Section \ref{sec:experiments}
demonstrates empirically that the approach suggested by the
theoretical results can produce good planning performance with a
flawed model, while reward models learned in the typical manner (or
even ``perfect'' reward functions) can lead to catastrophic planning
failure.

\section{Background}

\begin{figure}
  \centering
  \includegraphics[width=.75\linewidth]{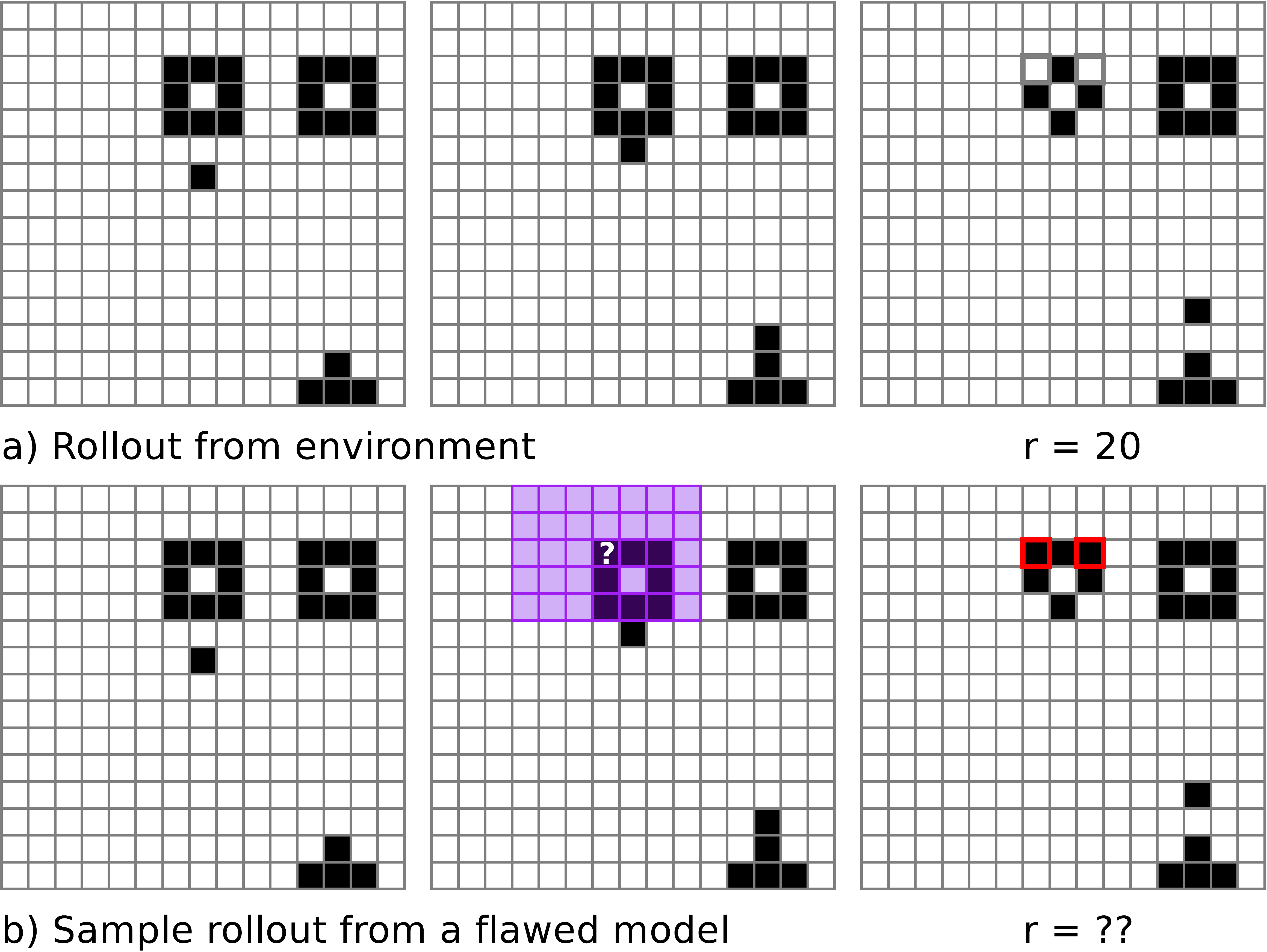}
  \caption{A flawed model may generate states for which the reward
    function is undefined.}
  \label{fig:shooterRollout}
\end{figure}

We focus on {\em Markov decision processes} (MDP). The environment's
initial state $s_1$ is drawn from a distribution
$\mu$. At each step $t$ the environment is in a state $s_t$. The agent
selects an action $a_t$ which causes the environment
to transition to a new state sampled from the transition distribution:
$s_{t+1} \sim P_{s_t}^{a_t}$.  The environment also emits a reward,
$R_{s_t}^{a_t}$. We assume that rewards are bounded within
$[0, M]$.

A {\em policy} $\pi$ specifies a way to behave in the MDP. Let
$\pi(a \mid s)$ be the probability that $\pi$ chooses action $a$ in
state $s$. For a sequence of actions $a_{1:t}$ let
$P(s' \mid s, a_{1:t}) = P_s^{a_{1:t}}(s')$ be the probability of
reaching $s'$ by starting in $s$ and taking the actions in the
sequence. For any state $s$, action $a$, and policy $\pi$, let
$D^t_{s, a, \pi}$ be the state-action distribution obtained after $t$
steps, starting with state $s$ and action $a$ and thereafter following
policy $\pi$. For a state action distribution $\xi$, let
$D^t_{\xi, \pi} = \expects_{(s, a) \sim \xi} D^t_{s, a, \pi}$.  We let
$\mathcal{S}$ be the set of states reachable in finite time by some
policy with non-zero probability. One may only observe the behavior
of $P$ and $R$ in states contained in $\mathcal{S}$.

The $T$-step {\em state-action value} of a policy, $Q^\pi_T(s, a)$
represents the expected discounted sum of rewards obtained by taking
action $a$ in state $s$ and executing $\pi$ for an additional $T-1$
steps:
$Q^\pi_T(s, a) = \sum_{t = 1}^{T}\gamma^{t-1} \expects_{(s', a') \sim
  D^t_{s, a, \pi}} R_{s'}^{a'}$.
Let the $T$-step {\em state value}
$V^\pi_T(s) = \expects_{a \sim \pi_s}[Q^\pi_T(s, a)]$. Let
$Q^\pi = Q^\pi_{\infty}$, and $V^\pi = V^\pi_{\infty}$. The agent's
goal will be to learn a policy $\pi$ that maximizes
$\expects_{s \sim \mu}[V^\pi(s)]$.

In MBRL we seek to learn a dynamics model $\hat{P}$, approximating
$P$, and a reward model $\hat{R}$, approximating $R$, and then to use
the combined model $(\hat{P}, \hat{R})$ to produce a policy via a
planning algorithm. We let $\hat{D}$, $\hat{Q}$, and $\hat{V}$
represent the corresponding quantities using the learned model. We
assume that $\hat{P}$ and $\hat{R}$ are defined over
$\hat{\mathcal{S}} \supseteq \mathcal{S}$; there may be states in
$\hat{\mathcal{S}}$ for which $P$ and $R$ are effectively undefined,
and it may not be known {\em a priori} which states these are.

Let $\mathcal{P}$ represent the {\em dynamics model class}, the set of
models the learning algorithm could possibly produce and
correspondingly let $\mathcal{R}$ be the {\em reward model class}. In
this work we are most interested in the common case that the dynamics
model is {\em misspecified}: there is no $\hat{P} \in \mathcal{P}$
that matches $P$ in every $s \in \mathcal{S}$. In this case
it is impossible to learn a perfectly accurate model; the agent must
make good decisions despite flaws in the learned model. The results in
this paper also permit the reward model to be similarly misspecified.

\subsection{Bounding Planning Performance}

For ease of analysis we focus our attention on the simple one-ply
Monte Carlo planning algorithm (one-ply MC), similar to the ``rollout
algorithm'' \citep{tesauro1996line}. For every state-action pair
$(s, a)$, the planner executes $N$ $T$-step sample rollouts using
$\hat{P}$, starting at $s$, taking action $a$, and then following a
{\em rollout policy} $\rho$. At each step of the rollout, $\hat{R}$
gives the reward. Let $\bar{Q}(s, a)$ be the average discounted return
of the rollouts starting with state $s$ and action $a$. For large $N$,
$\bar{Q}$ will closely approximate $\hat{Q}_T^{\rho}$
\citep{kakade2003sample}. The execution policy $\hat{\pi}$ will be
greedy with respect to $\bar{Q}$. We can place bounds on
the quality of $\hat{\pi}$.

For a policy $\pi$ and state-action distribution $\xi$, let
$\epsilon_{val}^{\xi, \pi, T}$ be the error in the $T$-step
state-action values the model assigns to the policy:
$\epsilon_{val}^{\xi, \pi, T} = \expect_{(s, a) \sim
  \xi}\big[|Q_T^{\pi}(s, a) - \hat{Q}_T^{\pi}(s, a)|\big]$. For a
state distribution $\mu$ and policy $\pi$ let $D_{\mu, \hat{\pi}}(s, a) =
\sum_{t=0}^{\infty} \gamma^{t} D_{\mu, \pi}^{t+1}(s, a)$.
The following is straightforwardly adapted from an
existing bound \cite{talvitie2015agnostic,talvitie2017self}.

\begin{lemma} \label{lem:mcvaluebound} Let $\bar{Q}$ be the value
  function returned by applying depth $T$ one-ply Monte Carlo to the
  model $\hat{P}$ with rollout policy $\rho$. Let $\hat{\pi}$ be
  greedy w.r.t. $\bar{Q}$. For any policy $\pi$ and state-distribution
  $\mu$,
  \begin{align*}
    \expect_{s \sim \mu}\big[V^\pi(s) - V^{\hat{\pi}}(s)\big] \le
    \frac{4}{1 - \gamma}\epsilon_{val}^{\xi, \rho, T}  +\epsilon_{mc}, 
    \end{align*}
    where $\xi(s, a) = \frac{1}{2} D_{\mu, \hat{\pi}}(s, a) + \frac{1}{4}
  D_{\mu, \pi}(s, a)
  + \frac{1}{4}\left((1 - \gamma) \mu(s)
    \hat{\pi}_s(a) + \gamma \sum_{z, b} D_{\mu, \pi}(z, b)
    P_{z}^{b}(s) \hat{\pi}_s(a) \right)$ and $\epsilon_{mc} = \frac{4}{1 - \gamma}\|\bar{Q} -
  \hat{Q}^{\rho}_T\|_\infty + \frac{2}{1 - \gamma} \|BV^{\rho}_T -
  V^{\rho}_T\|_{\infty}$ (here $B$ is the Bellman operator).
\end{lemma} 
The $\epsilon_{mc}$ term captures error due to properties of the
one-ply MC algorithm: error in the sample average $\bar{Q}$ and the
sub-optimality of the $T$-step value function with respect to
$\rho$. The $\epsilon_{val}^{\xi, \rho, T}$ term captures error due to
the model. The model's usefulness for planning is tied to the accuracy
of the value it assigns to the rollout policy. Thus, in order to
obtain a good plan $\hat{\pi}$, we aim to minimize
$\epsilon_{val}^{\xi, \rho, T}$.

\subsection{Error in the Dynamics Model}

If the reward function is known, a bound on
$\epsilon_{val}^{\xi, \rho, T}$ in terms of the one-step prediction error of
the dynamics model can be adapted from the work of
\citet{ross2012agnostic} .
\begin{lemma} \label{lem:onestepbound} 
For any policy $\pi$ and
  state-action distribution $\xi$,
\begin{align*}
\epsilon_{val}^{\xi, \pi, T} \le \frac{M}{1 - \gamma}\sum_{t = 1}^{T-1} (\gamma^{t} - 
\gamma^{T})\expect_{(s, a) \sim
  D^t_{\xi, \pi}}\big[\|P_s^a - \hat{P}_s^a\|_1\big].
\end{align*}
\end{lemma} 

Combining Lemmas \ref{lem:mcvaluebound} and \ref{lem:onestepbound}
yields an overall bound on control performance in terms of model
error. However, recent work \cite{talvitie2017self} offers a tighter
bound in a special case. Let the true dynamics $P$ be deterministic,
and let the rollout policy $\rho$ be {\em blind}
\cite{06icml-psr-exploration}; the action selected by $\rho$ is
conditionally independent of the current state, given the history of
actions. Then for any state-action distribution $\xi$, let
$H^t_{\xi, \rho}$ be the joint distribution over environment state,
model state, and action if a single action sequence is sampled from
$\rho$ and then executed in both the model and the environment. So, $H^1_{\xi, \rho}(s_1, z_1, a_1) =\xi(s_1, a_1) \text{ when } z_1 =
  s_1 \text{ (0 otherwise)}$
and for all $t \ge 2$,
\begin{align*}
H&^t_{\xi, \rho}(s_t, z_t, a_t) = \\
&\expect_{(s_1, a_1) \sim \xi}
\big[\textstyle\sum_{a_{2:t-1}} \rho(a_{2:t} \mid a_1)
P_{s_1}^{a_{1:t-1}}(s_t) \hat{P}_{s_1}^{a_{1:t-1}}(z_t)\big].
\end{align*}
Since $P$ is deterministic, let $\sigma_{s}^{a_{1:t}}$ be the unique
state that results from starting in state $s$ and taking the action
sequence $a_{1:t}$. Then \citet{talvitie2017self} offers the following
result:
\begin{theorem} \label{thm:tightness}
  If $P$ is deterministic, then for any blind policy $\rho$ and any
  state-action distribution $\xi$,
\begin{align}
\epsilon&_{val}^{\xi, \rho, T} \le\ M\sum_{t = 1}^{T} \gamma^{t-1}
                             \expect_{(s, a) \sim \xi}\big[\|D^t_{s, a,
                             \rho} - \hat{D}^t_{s, a, \rho}\|_1\big]\label{eq:multistep}\\
&\le
 2M \sum_{t = 1}^{T-1}\gamma^t
  \expect_{(s, z, a) \sim H^{t}_{\xi, \rho}} \big[1 -
  \hat{P}_z^a(\sigma^a_s)\big]\label{eq:hallucinated}\\
&\le \frac{2M}{1 - \gamma}\sum_{t = 1}^{T-1} (\gamma^{t} - 
\gamma^{T})\expect_{(s, a) \sim
  D^t_{\xi, \rho}}\big[1 - \hat{P}_s^a(\sigma^a_s)\big].\label{eq:onestep}
\end{align}
\end{theorem}
Inequality \ref{eq:onestep} is Lemma \ref{lem:onestepbound}
in the deterministic case, the bound in terms
of the one-step prediction error of $\hat{P}$. Inequality
\ref{eq:multistep} gives the bound in terms of the error in the discounted
distribution of states along $T$-step rollouts. Though this is the
tightest bound of the three, in practice it is difficult to optimize
this objective directly. Inequality \ref{eq:hallucinated} gives the
bound in terms of {\em hallucinated one-step error}, so called because
it considers the accuracy of the model's predictions based on states
generated from its own sample rollouts ($z$), rather than states
generated by the environment ($s$). 

To optimize hallucinated error, the model can be rolled out in
parallel with the environment, and trained to predict the next
environment state from each ``hallucinated'' state in the model
rollout. \citet{talvitie2017self} shows that this approach can
dramatically improve planning performance when the model class is
mispecified. Similar approaches have also had empirical success in
MBRL tasks \cite{talvitie2014model,venkatraman2016improved} and
sequence prediction tasks \cite{venkatraman2015improving,
  oh2015action,bengio2015scheduled}.

\citet{talvitie2017self} shows that the relative tightness of the
hallucinated error bound does not hold for general stochastic dynamics
or for arbitrary rollout policies. However, note that these
assumptions are not as limiting as they first appear. By far the most
common rollout policy chooses actions uniformly randomly, and is thus
blind. Furthermore, though $P$ is assumed to be deterministic, it is
also assumed to be too complex to be practically captured by
$\hat{P}$. From the agent's perspective, un-modeled complexity will
manifest as apparent stochasticity. For example
\citenoun{oh2015action} learned dynamics models of Atari 2600 games,
which are fully deterministic \citep{hausknecht2014neuroevolution};
human players often perceive them to be stochastic due to their
complexity. For the remainder of the paper we focus on the special
case of deterministic dynamics and blind rollout policies.

\section{Incorporating Reward Error}\label{sec:rwderr}

As suggested by \citet{talvitie2017self}, there is a straightforward
extension of Theorem \ref{thm:tightness} to account for reward error.
\begin{theorem}\label{thm:naivereward}
  If $P$ is deterministic, then for any blind policy $\rho$ and any
  state-action distribution $\xi$,
\begin{align}
\epsilon_{val}^{\xi, \rho, T} &\le \sum_{t = 1}^T
    \gamma^{t-1}\expect_{(s', a') \sim D^t_{\xi, \rho}}\big[\big|R_{s'}^{a'}
  - \hat{R}_{s'}^{a'}\big|\big]\label{eq:multistepplusr}\\
&\hspace{.1in} + M\sum_{t = 1}^{T} \gamma^{t-1}
                             \expect_{(s, a) \sim \xi}\big[\|D^t_{s, a,
                             \rho} - \hat{D}^t_{s, a, \rho}\|_1\big]\nonumber\\
&\le \sum_{t = 1}^T
    \gamma^{t-1}\expect_{(s', a') \sim D^t_{\xi, \rho}}\big[\big|R_{s'}^{a'}
  - \hat{R}_{s'}^{a'}\big|\big]\label{eq:hallucinatedplusr}\\
&\hspace{.1in} +
 2M \sum_{t = 1}^{T-1}\gamma^t
  \expect_{(s, z, a) \sim H^{t}_{\xi, \rho}} \big[1 -
  \hat{P}_z^a(\sigma^a_s)\big]\nonumber\\
&\le \sum_{t = 1}^T
    \gamma^{t-1}\expect_{(s', a') \sim D^t_{\xi, \rho}}\big[\big|R_{s'}^{a'}
  - \hat{R}_{s'}^{a'}\big|\big]\\
&\hspace{.1in} + \frac{2M}{1 - \gamma}\sum_{t = 1}^{T-1} (\gamma^{t} - 
\gamma^{T})\expect_{(s, a) \sim
  D^t_{\xi, \rho}}\big[1 - \hat{P}_s^a(\sigma^a_s)\big]\nonumber.
\end{align}
\end{theorem}
\begin{proof}
  The derivation of inequality \ref{eq:multistepplusr} is below. The
  rest follow immediately from Theorem \ref{thm:tightness}.
  \begin{align*}
    &\epsilon_{val}^{\xi, \rho, T} =\expect_{(s, a) \sim \xi} \big[|Q_T^\rho(s, a) - \hat{Q}_T^\rho(s,a)|\big]&\\
    &= \expect_{(s, a) \sim \xi} \Bigg[\bigg| \sum_{t = 1}^T \gamma^{t-1}\sum_{(s',
      a')}\Big(D^t_{s, a, \rho}(s', a')R_{s'}^{a'}&\\
    &\hspace{1.8in}- \hat{D}^t_{s, a, \rho}(s', a')\hat{R}_{s'}^{a'} \Big) \bigg| \Bigg]&\\
    &= \expect_{(s, a) \sim \xi} \Bigg[\bigg| \sum_{t = 1}^T \gamma^{t-1}\sum_{(s',
      a')}\Big(D^t_{s, a, \rho}(s', a')R_{s'}^{a'}&\\
    &\hspace{.8in}- D^t_{s, a, \rho}(s', a')\hat{R}_{s'}^{a'} + D^t_{s, a, \rho}(s', a')\hat{R}_{s'}^{a'}&\\
    &\hspace{1.8in}- \hat{D}^t_{s, a, \rho}(s', a')\hat{R}_{s'}^{a'} \Big) \bigg| \Bigg]&\\
    &= \expect_{(s, a) \sim \xi} \Bigg[\bigg| \sum_{t = 1}^T \gamma^{t-1}\sum_{(s',
      a')}\Big(D^t_{s, a, \rho}(s', a')(R_{s'}^{a'} -
      \hat{R}_{s'}^{a'})&\\
    &\hspace{0.9in}+ (D^t_{s, a, \rho}(s', a') - \hat{D}^t_{s, a, \rho}(s', a'))\hat{R}_{s'}^{a'} \Big) \bigg| \Bigg]&
\end{align*}
\begin{align*}
    &\le \sum_{t = 1}^T
      \gamma^{t-1}\expect_{(s', a') \sim D^t_{\xi, \rho}}\big|R_{s'}^{a'} - \hat{R}_{s'}^{a'}\big|&\\
    &\hspace{0.7in}+ M\sum_{t = 1}^T
      \gamma^{t-1}\expect_{(s, a) \sim \xi} \big[\big\|D^t_{s, a, \rho} - \hat{D}^t_{s, a, \rho}\big\| \big],&
  \end{align*}
which gives the result.
\end{proof}

As is typical, these bounds break the value error into two parts:
reward error and dynamics error. The reward error measures the
accuracy of the reward model in environment states encountered by
policy $\rho$. The dynamics error measures the probability that the
model will generate the correct states in rollouts, effectively
assigning maximum reward error ($M$) when the dynamics model generates
incorrect states. This view corresponds to common MBRL practice:
train the dynamics model to assign high probability to
correct environment states and the reward model to accurately map environment
states to rewards. However, as discussed in Section \ref{sec:example},
these bounds are overly conservative (and thus loose): generating an
erroneous state need not be catastrophic if the associated reward is
still reasonable. We can derive a bound that accounts for this.

\begin{theorem}\label{thm:hrwd}
  If $P$ is deterministic, then for any blind policy $\rho$ and any
  state-action distribution $\xi$,
  \begin{align*}
    \epsilon_{val}^{\xi, \rho, T} &\le \sum_{t=1}^T \gamma^{t-1} \expects_{(s, z, a) \sim H^t_{\xi, \rho}}
      \big[\big| R_{s}^{a} - \hat{R}_{z}^{a}\big|\big].
  \end{align*}
\end{theorem}
\begin{proof}
  \begin{align*}
&\epsilon_{val}^{\xi, \rho, T} = \expects_{(s_1, a_1) \sim \xi}\big[\big|Q_\rho^T(s_1, a_1) - \hat{Q}_\rho^T(s_1, a_1)\big|\big]&\\
&=\expect_{(s_1, a_1) \sim \xi}\Bigg[\Bigg|\sum_{t=1}^T \gamma^{t-1}
\sum_{s_t, a_t} D^t_{s_1, a_1, \rho}(s_t, a_t)R_{s_t}^{a_t}&\\
&\hspace{1.6in} - \sum_{z_t, a_t}
\hat{D}^t_{s_1, a_1, \rho}(z_t, a_t) \hat{R}_{z_t}^{a_t}\Bigg|\Bigg]&\\
&=\expect_{(s_1, a_1) \sim \xi}\Bigg[\Bigg|\big(R_{s_1}^{a_1} -
  \hat{R}_{s_1}^{a_1}\big) + \sum_{t=2}^T \gamma^{t-1}
\sum_{a_{2:t}} \rho(a_{2:t} \mid a_1)&\\
&\hspace{0.5in} \bigg( \sum_{s_t}
  P_{s_1}^{a_{1:t-1}}(s_t)R_{s_t}^{a_t} - \sum_{z_t}
\hat{P}_{s_1}^{a_{1:t-1}}(z_t) \hat{R}_{z_t}^{a_t}\bigg)\Bigg|\Bigg].&
\end{align*}
Now note that for $t \ge 2$,
\begin{align*}
  &\sum_{s_t} P_{s_1}^{a_{1:t-1}}(s_t)R_{s_t}^{a_t} - \sum_{z_t}
  \hat{P}_{s_1}^{a_{1:t-1}}(z_t) \hat{R}_{z_t}^{a_t}\\
&= \sum_{s_t}
P_{s_1}^{a_{1:t-1}}(s_t)R_{s_t}^{a_t}\sum_{z_t}
\hat{P}_{s_1}^{a_{1:t-1}}(z_t)&\\
&\hspace{1.1in} - \sum_{z_t}
\hat{P}_{s_1}^{a_{1:t-1}}(z_t) \hat{R}_{z_t}^{a_t}\sum_{s_t}
P_{s_1}^{a_{1:t-1}}(s_t) &\\
&=\sum_{s_t, z_t} P_{s_1}^{a_{1:t-1}}(s_t)
  \hat{P}_{s_1}^{a_{1:t-1}}(z_t)\big( R_{s_t}^{a_t} -
  \hat{R}_{z_t}^{a_t}\big)&
\end{align*}
Thus
\begin{align*}
&\epsilon_{val}^{\xi, \rho, T} \le \expect_{(s_1, a_1) \sim
                                \xi}\Bigg[\big|R_{s_1}^{a_1}-\hat{R}_{s_1}^{a_1}\big|
                                + \sum_{t=2}^T \gamma^{t-1}&\\
&\sum_{a_{2:t}} \rho(a_{2:t} \mid a_1)\sum_{s_t, z_t} P_{s_1}^{a_{1:t-1}}(s_t) \hat{P}_{s_1}^{a_{1:t-1}}(z_t)\big| R_{s_t}^{a_t} - \hat{R}_{z_t}^{a_t}\big|\Bigg]&\\
&= \sum_{t=1}^T \gamma^{t-1} \expect_{(s, z, a) \sim H^t_{\xi, \rho}} \big[\big| R_{s}^{a} - \hat{R}_{z}^{a}\big|\big].&\qedhere
\end{align*}
\end{proof}

Similar to the hallucinated one-step error for the dynamics model
(inequality \ref{eq:hallucinated}), Theorem \ref{thm:hrwd} imagines
that the model and the environment are rolled out in parallel. It
measures the error between the rewards generated in the model rollout
and the rewards in the corresponding steps of the environment
rollout. We call this the {\em hallucinated reward error}. However,
unlike the bounds in Theorem \ref{thm:naivereward}, which are focused on
the model placing high probability on ``correct'' states, the
hallucinated reward error may be small even if the state sequence
sampled from the dynamics model is ``incorrect'', as long as the
sequence of {\em rewards} is similar. As such, we can show that this
bound is tighter than inequality \ref{eq:hallucinatedplusr} and thus
more closely related to planning performance.

\begin{theorem}\label{thm:tighter}
  If $P$ is deterministic, then for any blind policy $\rho$ and any
  state-action distribution $\xi$,
  \begin{align*}
    \sum_{t=1}^T& \gamma^{t-1} \expects_{(s, z, a) \sim H^t_{\xi, \rho}}
      \big[\big| R_{s}^{a} - \hat{R}_{z}^{a}\big|\big]\\
                                  &\le \sum_{t = 1}^T
                                    \gamma^{t-1}\expect_{(s', a') \sim D^t_{\xi, \rho}}\big[\big|R_{s'}^{a'}
                                    - \hat{R}_{s'}^{a'}\big|\big]\\
                                  &\hspace{.6in} +
                                    2M \sum_{t = 1}^{T-1}\gamma^t
                                    \expect_{(s, z, a) \sim H^{t}_{\xi, \rho}} \big[1 -
                                    \hat{P}_z^a(\sigma^a_s)\big].
  \end{align*}
\end{theorem}
\begin{proof}
  \begin{align*}
    \sum_{t=1}^T& \gamma^{t-1} \expect_{(s, z, a) \sim H^t_{\xi, \rho}}
    \big[\big| R(s, a) - \hat{R}(z, a)\big|\big]\\
    &= \sum_{t=1}^T \gamma^{t-1} \sum_{s, z, a} H^t_{\xi, \rho}(s, z, a)\big| R(s, a) - \hat{R}(z, a)\big|\\
    &= \sum_{t=1}^T \gamma^{t-1} \sum_{s, a} H^t_{\xi, \rho}(s,
      s, a) \big|R(s, a) - \hat{R}(s, a)\big|\\
                &\hspace{.1in} + \sum_{t=1}^T \gamma^{t-1}\sum_{s,
      z \ne s, a} H^t_{\xi, \rho}(s, z, a)\big| R(s, a) - \hat{R}(z,
      a)\big|.
  \end{align*}
  This breaks the expression into two terms. Now consider the first term:
  \begin{align}
    \sum_{t=1}^T& \gamma^{t-1} \sum_{s, a} H^t_{\xi, \rho}(s,
    s, a) \big|R(s, a) - \hat{R}(s, a)\big|\nonumber\\
    &\le \sum_{t=1}^T \gamma^{t-1}\sum_{s, a} D^t_{\xi, \rho}(s, a)
      \big|R(s, a) - \hat{R}(s, a)\big|. \label{eq:term1}
  \end{align}
  Now consider the second term:
  \begin{align*}
    \sum_{t=1}^T& \gamma^{t-1}\sum_{s, z \ne s, a} H^t_{\xi, \rho}(s, z, a)\big| R(s, a) - \hat{R}(z,
      a)\big|\\
    &\le M \sum_{t=1}^{T} \gamma^{t-1}\sum_{s,
      z \ne s, a} H^{t}_{\xi, \rho}(s, z, a).
  \end{align*}
  Recall that $H^1_{\xi, \rho}(s, z, a) = 0$ if $s \ne z$. Thus,
  \begin{align}
    &M \sum_{t=1}^{T} \gamma^{t-1}\sum_{s,
      z \ne s, a} H^{t}_{\xi, \rho}(s, z, a)\nonumber\\
                    &= M \sum_{t=1}^{T-1} \gamma^{t}\sum_{s, z \ne s,
                                         a} H^{t+1}_{\xi, \rho}(s, z,
                                         a)\nonumber\\ 
    &= M \sum_{s', z', a'}
      \Bigg(\sum_{s, z \ne s} P^{a'}_{s'}(s) \hat{P}^{a'}_{z'}(z)
      \Bigg) \sum_{t=1}^{T-1} \gamma^{t} H^{t}_{\xi, \rho}(s', z', a')\nonumber\\
    &= M \sum_{s', z', a'}
      \big(1 - \hat{P}^{a'}_{z'}(\sigma^{a'}_{s'})\big) \sum_{t=1}^{T-1} \gamma^{t} H^{t}_{\xi, \rho}(s', z', a')\nonumber\\
    &= M \sum_{t=1}^{T-1} \gamma^t
      \expect_{(s, z, a) \sim H^t_{\xi, \rho}} [1 -
      \hat{P}^{a}_{z}(\sigma^{a}_{s})].  \label{eq:term2}
  \end{align}
  Combining lines \ref{eq:term1} and \ref{eq:term2} yields the result.
\end{proof}

The next section discusses the practical and conceptual implications of this result
for MBRL algorithms and extends an existing MBRL algorithm to
incorporate this insight.

\section{Implications for MBRL}\label{sec:implications}

This is not the first observation that the reward model should be
specialized to the dynamics model . \citet{sorg2010internal} argued as
we have that when the model or planner are limited in some way, reward
functions other than the true reward may lead to better planning
performance. Accordingly, policy gradient approaches have been
employed to learn reward functions for use with online planning
algorithms, providing a benefit even when the reward function is known
\cite{sorg2010reward,sorg2011optimal,bratman2012strong,guo2016deep}. \citet{tamar2016value}
take this idea to its logical extreme, treating the entire model and
even the planning algorithm itself as a policy parameterization,
adapting them to directly improve control performance rather than to
minimize any measure of prediction error. Though appealing in its
directness, this approach offers little theoretical insight into what
makes a model useful for planning. Furthermore, there are advantages
to optimizing quantities other than control performance; this allows
the model to exploit incoming data even when it is unclear how to
improve the agent's policy (for instance if the agent has seen little
reward). Theorem \ref{thm:hrwd} provides more specific guidance about
how to choose amongst a set of flawed models. Rather than attempting
to directly optimize control performance, this result suggests that we
can take advantage of reward error signals while still offering
guarantees in terms of control performance.

It is notable that, unlike Theorem \ref{thm:naivereward}, Theorem
\ref{thm:hrwd} does not contain a term measuring dynamics
error. Certainly the dynamics model is implicitly important; for some
choices of $\hat{P}$ the hallucinated reward error can be made very
small while for others it may be irreducibly high (for instance if
$\hat{P}$ simply loops on a single state). Nevertheless, low
hallucinated reward error does not require that the dynamics model
place high probability on ``correct'' states. In fact, it may be that
dynamics entirely unrelated to the environment yield the best reward
predictions. This intriguingly suggests that the dynamics
model and reward model parameters could be adapted together to
optimize hallucinated reward error. Arguably, the recently introduced Predictrons
\cite{silver2017predictron} and Value Prediction Networks
\cite{oh2017value} are attempts to do just this -- they adapt the
model's dynamics solely to improve reward prediction. We can see
Theorem \ref{thm:hrwd} as theoretical support for these approaches and
encouragement of more study in this direction. Still, in practice it
may be much harder to learn to predict reward sequences than state
sequences, especially when the reward signal is sparse. Also, the
relationship between reward prediction error and dynamics model parameters
can be highly complex, which may make theoretical performance
guarantees difficult.

Another possible interpretation of Theorem \ref{thm:hrwd} is that the
reward model should be customized to the dynamics model. That is, if
we hold the dynamics model fixed, then the result gives a clear
objective for the reward model. Theorem \ref{thm:tighter} suggests an
algorithmic structure where the dynamics model is trained via its own
objective, and the reward model is then trained to minimize
hallucinated error with respect to the learned dynamics model. The
clear downside of this approach is that it will not in general find
the best combination of dynamics model and reward model; it could be
that a less accurate dynamics model results in lower hallucinated
reward error. The advantage is that it allows us to effectively exploit the
prediction error signal for the dynamics model and removes the
circular dependence between the dynamics model and the reward model.

In this paper we explore this avenue by extending the existing
Hallucinated DAgger-MC algorithm \cite{talvitie2017self}. Because the
resulting algorithm is very similar to the original, we leave a
detailed description and analysis to the appendix and here focus on
key, high-level points. Section \ref{sec:experiments} presents
empirical results illustrating the practical impact of training the reward model
to minimize hallucinated error.

\subsection{Hallucinated DAgger-MC with Reward Learning}\label{sec:hdaggermc}

The ``Data Aggregator'' (DAgger) algorithm \citep{ross2012agnostic}
was the first practically implementable MBRL algorithm with
performance guarantees agnostic to the model class. It did, however,
require that the planner be near optimal. DAgger-MC
\citep{talvitie2015agnostic} relaxed this assumption, accounting for
the limitations of a particular suboptimal planner (one-ply
MC). Hallucinated DAgger-MC (or H-DAgger-MC) \cite{talvitie2017self}
altered DAgger-MC to optimize the hallucinated error, rather than the
one-step error. All of these algorithms were presented under the
assumption that the reward function was known {\em a priori}. As we
will see in Section \ref{sec:experiments}, the reward function cannot
be ignored; even when the reward function is given, these algorithms
can fail catastrophically due to the interaction between the reward
function and small errors in the dynamics model.

At a high level, H-DAgger-MC proceeds in iterations. In each iteration
a batch of data is gathered by sampling state-action pairs using a
mixture of the current plan and an ``exploration distribution'' (to
ensure that important states are visited, even if the plan would not
visit them). The rollout policy is used to generate parallel rollouts
in the environment and model from these sampled state-action pairs,
which form the training examples (with model state as input and
environment state as output). The collected data is used to update the
dynamics model, which is then used to produce a new plan to be used in
the next iteration. We simply augment H-DAgger-MC, adding a reward
learning step to each iteration (rather than assuming the reward is
given). In each rollout, training examples mapping ``hallucinated''
model states to the real environment rewards are collected and used to
update the reward model.

The extended H-DAgger-MC algorithm offers theoretical guarantees
similar to those of the original algorithm. Essentially, if
\begin{itemize}
\item the exploration distribution is similar to the state visitation
  distribution of a good policy, 
\item $\epsilon_{mc}$ is small,
\item the learning algorithms for the dynamics model and reward model
  are both no-regret, and
\item the reward model class $\mathcal{R}$ contains a low hallucinated
  reward error model with respect to the lowest hallucinated
  prediction error model in $\mathcal{P}$,
\end{itemize}
then in the limit H-DAgger-MC will produce a good policy. As discussed
in Section \ref{sec:implications}, this does {\em not} guarantee that
H-DAgger-MC will find the best performing combination of dynamics
model and reward model, since the training of the dynamics model does
not take hallucinated reward error into account. It is, however, an
improvement over the original H-DAgger-MC result in that good
performance can be assured even if there is no low error
dynamics model in $\mathcal{P}$, as long as there is a low error
reward model in $\mathcal{R}$.

For completeness' sake, a more detailed description and analysis of the
algorithm can be found in the appendix. Here we turn to an empirical
evaluation of the algorithm.

\section{Experiments}\label{sec:experiments}

In this section we illustrate the impact of optimizing hallucinated
reward error in the Shooter example described in Section
\ref{sec:intro} using both DAgger-MC and H-DAgger-MC\footnote{Source
  code for these experiments is available at \url{http://github.com/etalvitie/hdaggermc}.}. The
one-ply MC planner used 50 uniformly random rollouts of depth 20 per
action at every step. The exploration distribution was generated by
following the optimal policy with $(1 - \gamma)$ probability of
termination at each step. The discount factor was $\gamma = 0.9$. In
each iteration 500 training rollouts were generated and the resulting
policy was evaluated in an episode of length 30. The discounted return
obtained by the policy in each iteration is reported, averaged over 50
trials.

The dynamics model for each pixel was learned using Context Tree
Switching \citep{veness2012context}, similar to the FAC-CTW algorithm
\citep{veness11montecarlo}. At each position the model takes as input
the values of the pixels in a $w \times h$ neighborhood around the
position in the previous timestep. Data was shared across all
positions. The reward was approximated with a linear function for each
action, learned via stochastic weighted gradient descent. The feature
representation contained a binary feature for each possible
$3 \times 3$ configuration of pixels at each position. This
representation admits a perfectly accurate reward model. The
qualitative observations presented in this section were robust to a
wide range of choices of step size for gradient descent. Here, in each
experiment the best performing step size for each approach is selected
from 0.005, 0.01, 0.05, 0.1, and 0.5.

In the experiments a practical alteration has been made to the
H-DAgger-MC algorithm. H-DAgger-MC requires an ``unrolled''
dynamics model (with a separate model for each step of the rollout,
each making predictions based on the output of the previous
model). While this is important for H-DAgger-MC's theoretical
guarantees, \citet{talvitie2017self} found empirically that a single
dynamics model for all steps could be learned, provided that the training rollouts
had limited depth. Following \citet{talvitie2017self}, in the first 10
iterations only the first example from each training rollout is added
to the dynamics model dataset; thereafter only the first two examples
are added. The entire rollout was used to train the reward
model. DAgger-MC does not require an unrolled dynamics model or
truncated training rollouts and was
implemented as originally presented, with
a single dynamics model and full training rollouts \cite{talvitie2015agnostic}. 

\subsection{Results}

We consider both DAgger-MC and H-DAgger-MC with a perfect reward
model, a reward model trained only on environment states during
rollouts, and a reward model trained on ``hallucinated'' states as described in
Section \ref{sec:hdaggermc}. The perfect reward model is one that
someone familiar with the rules of the game would likely specify; it
simply checks for the presence of explosions in the three target
positions and gives the appropriate value if an explosion is present
or 0 otherwise (subtracting 1 if the action is ``shoot''). Results are
presented in three variations on the Shooter problem.

\begin{figure*}
  \centering
  \subfloat[No model limitations]{\includegraphics[width=0.32\textwidth]{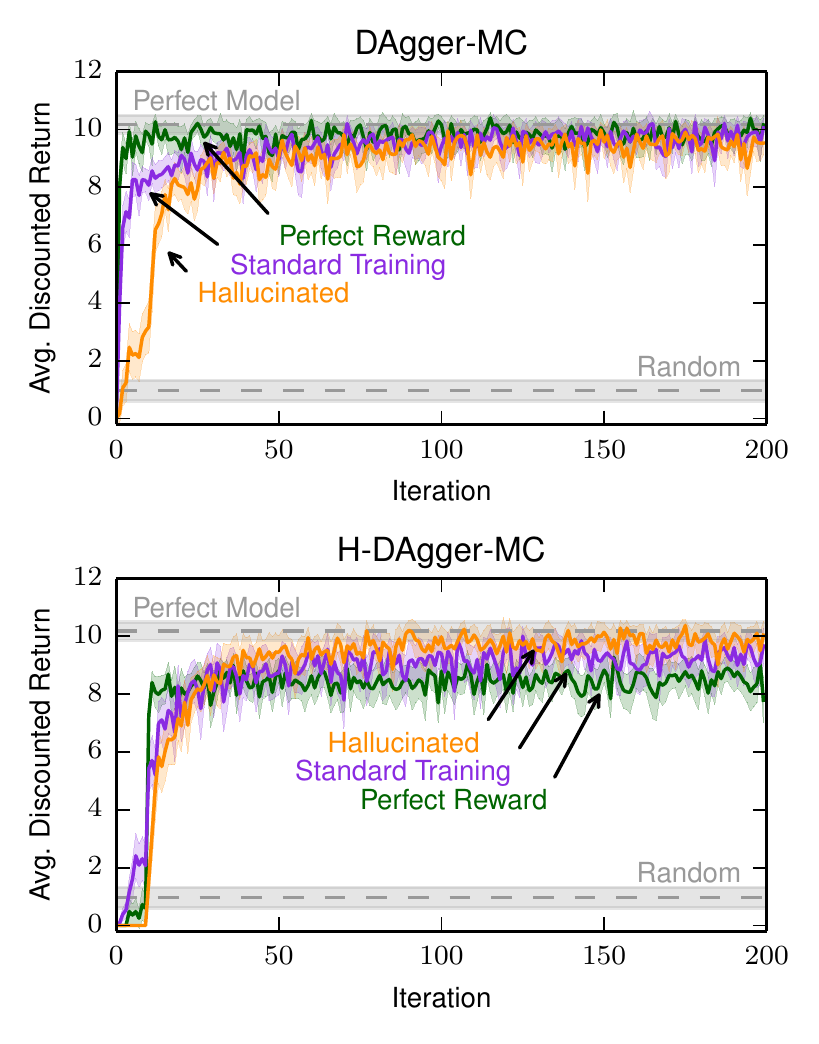}}\hfill
  \subfloat[Moving bullseyes (2nd-order Markov)]{\includegraphics[width=0.32\textwidth]{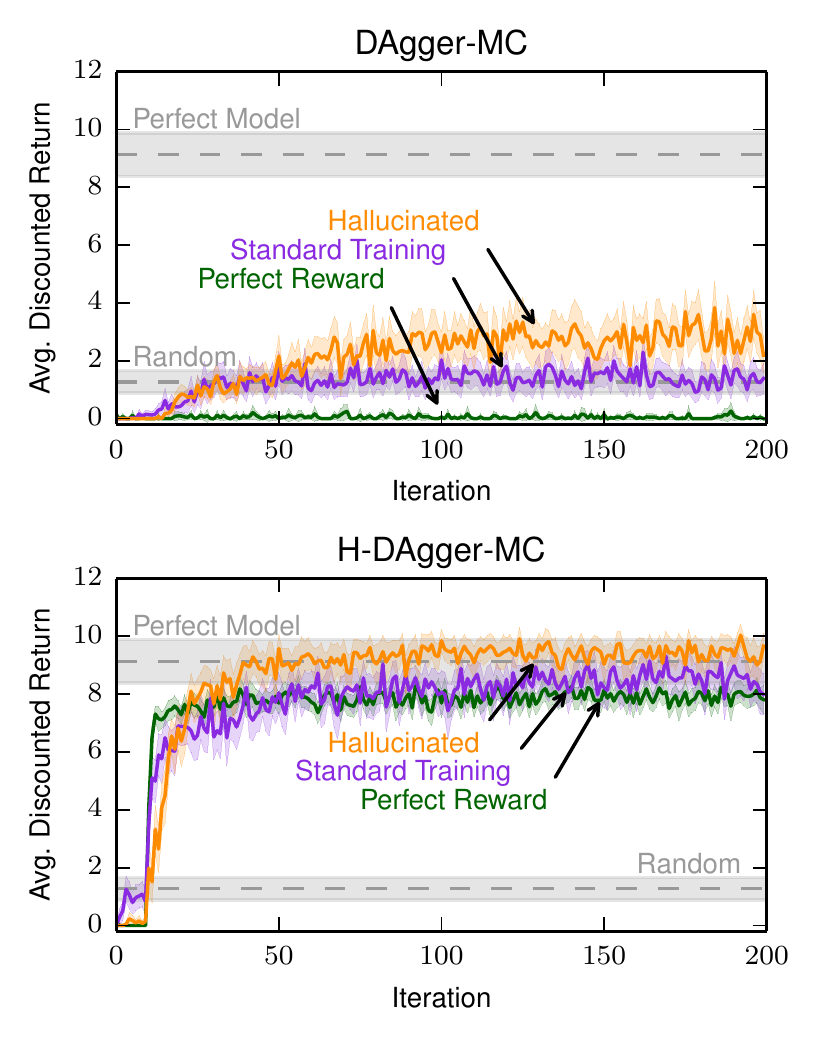}}\hfill
  \subfloat[Pixel models use $5 \times 7$ neighborhood]{\includegraphics[width=0.32\textwidth]{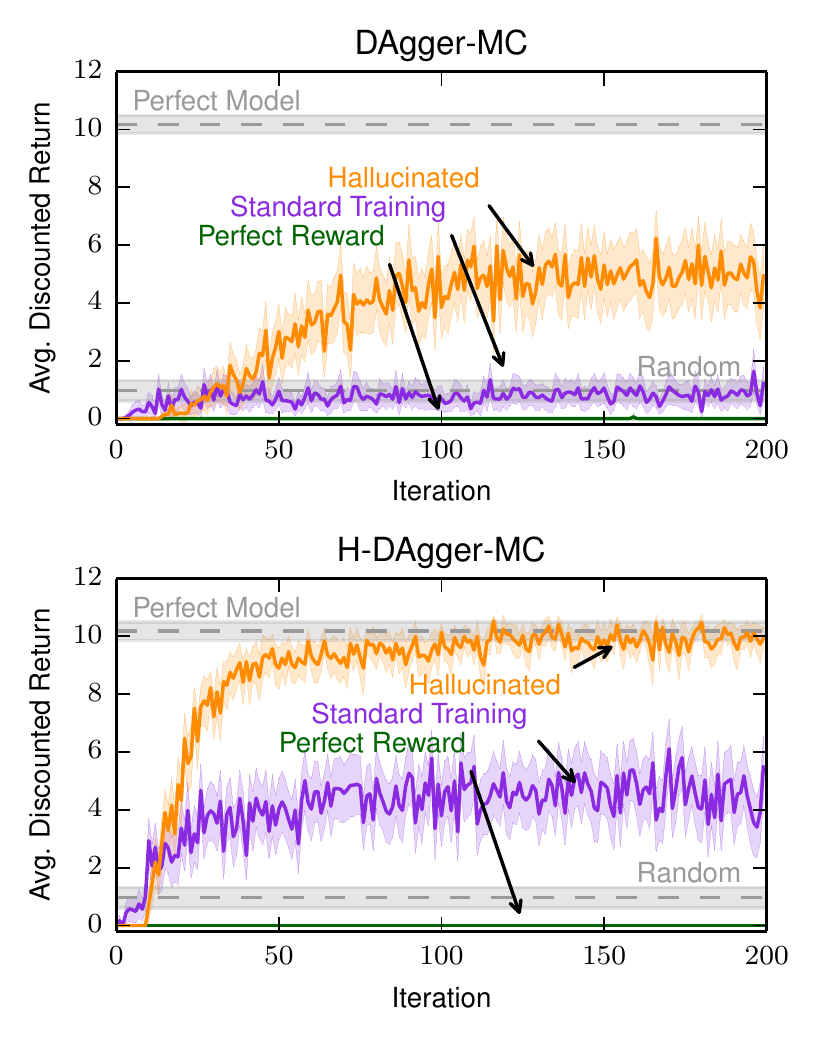}}\hfill
  \caption{Performance of DAgger-MC and H-DAgger-MC in three
    variations on the Shooter domain.}
  \label{fig:shooterResults}
\end{figure*}

\subsubsection{No Model Limitations}

In the first experiment we apply these algorithms to Shooter, as
described in Section \ref{sec:intro}. Here, the dynamics model uses a
$7 \times 7$ neighborhood, which is sufficient to make perfectly
accurate predictions for every pixel. Figure \ref{fig:shooterResults}a shows the
discounted return of the policies generated by DAgger-MC and
H-DAgger-MC, averaged over 50 independent trials. The shaded region
surrounding each curve represents a 95\% confidence interval. The gray
line marked ``Random'' shows the average discounted return of the
uniform random policy (with a 95\% confidence interval). The gray line
marked ``Perfect Model'' shows the average discounted return of the
one-ply MC planner using a perfect model.

Unsurprisingly, the performance DAgger-MC is comparable to that of
planning with the perfect model. As observed by
\citet{talvitie2017self}, with the perfect reward model H-DAgger-MC
performs slightly worse than DAgger-MC; the dynamics model in
H-DAgger-MC receives noisier data and is thus less
accurate. Interestingly, we can now see that the learned reward model
yields better performance than the perfect reward model, even without
hallucinated training! The perfect reward model relies on specific
screen configurations that are less likely to appear in flawed sample
rollouts, but the learned reward model generalizes to screens not seen
during training. Of course, it is coincidental that this
generalization is beneficial; under standard training the reward model
is only trained in environment states, giving no guidance in erroneous
model states. Hallucinated training specifically trains the reward
model to make reasonable predictions during model rollouts, so it
yields better performance, comparable with that of DAgger-MC. Thus we
see that learning the reward function in this way mitigates a
shortcoming of H-DAgger-MC, making it more effective in practice when
a perfectly accurate model can be learned.

\subsubsection{Failure of the Markov Assumption}

Next we consider a version of shooter presented by
\citet{talvitie2017self} in which the bullseye in each target moves
from side to side, making the environment second-order Markov. Because
the model is Markov, it cannot accurately predict the movement of the
bullseyes, though the representation is sufficient to accurately
predict every other pixel.

Figure \ref{fig:shooterResults}b shows the results. As
\citet{talvitie2017self} observed, DAgger-MC fails catastrophically in
this case. Though the model's limitation only prevents it from
accurately predicting the bullseyes, the resulting errors compound
during rollouts, quickly rendering them useless. As previously
observed, H-DAgger-MC performs much better, as it trains the model to
produce more stable rollouts. In both cases we see again that the
learned reward models outperform the perfect reward model, and
hallucinated reward training yields the best performance, even
allowing DAgger-MC to perform better than the random policy.

\subsubsection{Flawed Factored Structure}

We can see the importance of hallucinated reward training even more
clearly when we consider the original Shooter domain (with static
bullseyes), but limit the size of the neighborhood used to predict each
pixel, as described in Section \ref{sec:example}. Figure
\ref{fig:shooterResults}c shows the results. Once again DAgger-MC
fails. Again we see that the learned reward models yield better
performance than the perfect reward function, and that hallucinated
training guides the reward model to be useful for planning, despite
the flaws in the dynamics model.

In this case, we can see that H-DAgger-MC {\em also} fails when
combined with the perfect reward model, and performs poorly with the
reward model trained only on environment states. Hallucinated training
helps the dynamics model produce stable sample rollouts, but does not
correct the fundamental limitation: the dynamics model cannot
accurately predict the shape of the explosion when a target is hit. As
a result, a reward model that bases its predictions only the
explosions that occur in the environment will consistently fail to
predict reward when the agent hits a target in sample
rollouts. Hallucinated training, in contrast, specializes the reward
model to the flawed dynamics model, allowing for performance
comparable to planning with a perfect model.

\section{Conclusion}

This paper has introduced hallucinated reward error, which measures
the extent to which the rewards in a sample rollout from the model
match the rewards in a parallel rollout from the environment. Under
some conditions, this quantity is more tightly related to control
performance than the more traditional measure of model quality (reward
error in environment states plus error in state
transition). Empirically we have seen that when the dynamics model is
flawed, reward functions learned in the typical manner and even 
``perfect'' reward functions given {\em a priori} can lead to
catastrophic planning failure. When the reward function
is trained to minimize hallucinated reward error, it specifically
accounts for the model's flaws, significantly improving performance.

\section*{Acknowledgements}

This work was supported by NSF grant IIS-1552533. Thanks also to
Michael Bowling for his valuable input and to Joel Veness for his
freely available FAC-CTW and CTS implementations (\url{http:
//jveness.info/software/}).

\bibliography{hallucinatedReward_arXiv2018}

\begin{thebibliography}{27}
\providecommand{\natexlab}[1]{#1}
\providecommand{\url}[1]{\texttt{#1}}
\expandafter\ifx\csname urlstyle\endcsname\relax
  \providecommand{\doi}[1]{doi: #1}\else
  \providecommand{\doi}{doi: \begingroup \urlstyle{rm}\Url}\fi

\bibitem[Abbeel et~al.(2007)Abbeel, Coates, Quigley, and
  Ng]{abbeel2007application}
Abbeel, P., Coates, A., Quigley, M., and Ng, A.~Y.
\newblock An application of reinforcement learning to aerobatic helicopter
  flight.
\newblock In \emph{Advances in Neural Information Processing Systems 20
  (NIPS)}, pp.\  1--8, 2007.

\bibitem[Bellemare et~al.(2014)Bellemare, Veness, and
  Talvitie]{bellemare2014skip}
Bellemare, M.~G., Veness, J., and Talvitie, E.
\newblock Skip context tree switching.
\newblock In \emph{Proceedings of the 31st International Conference on Machine
  Learning (ICML)}, pp.\  1458--1466, 2014.

\bibitem[Bengio et~al.(2015)Bengio, Vinyals, Jaitly, and
  Shazeer]{bengio2015scheduled}
Bengio, S., Vinyals, O., Jaitly, N., and Shazeer, N.
\newblock Scheduled sampling for sequence prediction with recurrent neural
  networks.
\newblock In \emph{Advances in Neural Information Processing Systems 28
  (NIPS)}, pp.\  1171--1179, 2015.

\bibitem[Bowling et~al.(2006)Bowling, McCracken, James, Neufeld, and
  Wilkinson]{06icml-psr-exploration}
Bowling, M., McCracken, P., James, M., Neufeld, J., and Wilkinson, D.
\newblock Learning predictive state representations using non-blind policies.
\newblock In \emph{Proceedings of the 23rd International Conference on Machine
  Learning (ICML)}, pp.\  129--136, 2006.

\bibitem[Bratman et~al.(2012)Bratman, Singh, Sorg, and
  Lewis]{bratman2012strong}
Bratman, J., Singh, S., Sorg, J., and Lewis, R.
\newblock Strong mitigation: Nesting search for good policies within search for
  good reward.
\newblock In \emph{Proceedings of the 11th International Conference on
  Autonomous Agents and Multiagent Systems (AAMAS)}, pp.\  407--414, 2012.

\bibitem[Chiappa et~al.(2017)Chiappa, Racani{\`{e}}re, Wierstra, and
  Mohamed]{chiappa2017recurrent}
Chiappa, S., Racani{\`{e}}re, S., Wierstra, D., and Mohamed, S.
\newblock Recurrent environment simulators.
\newblock In \emph{Proceedigns of the International Conference on Learning
  Representations (ICLR)}, 2017.

\bibitem[Ebert et~al.(2017)Ebert, Finn, Lee, and Levine]{ebert2017self}
Ebert, F., Finn, C., Lee, A.~X., and Levine, S.
\newblock Self-supervised visual planning with temporal skip connections.
\newblock In \emph{Proceedings of the 1st Annual Conference on Robot Learning
  (CoRL)}, volume~78 of \emph{Proceedings of Machine Learning Research (PMLR)},
  pp.\  344--356, 2017.

\bibitem[Guo et~al.(2016)Guo, Singh, Lewis, and Lee]{guo2016deep}
Guo, X., Singh, S.~P., Lewis, R.~L., and Lee, H.
\newblock Deep learning for reward design to improve monte carlo tree search in
  {ATARI} games.
\newblock In \emph{Proceedings of the Twenty-Fifth International Joint
  Conference on Artificial Intelligence (IJCAI)}, pp.\  1519--1525, 2016.

\bibitem[Hausknecht et~al.(2014)Hausknecht, Lehman, Miikkulainen, and
  Stone]{hausknecht2014neuroevolution}
Hausknecht, M., Lehman, J., Miikkulainen, R., and Stone, P.
\newblock A neuroevolution approach to general atari game playing.
\newblock \emph{IEEE Transactions on Computational Intelligence and AI in
  Games}, 6\penalty0 (4):\penalty0 355--366, 2014.

\bibitem[Kakade(2003)]{kakade2003sample}
Kakade, S.~M.
\newblock \emph{On the sample complexity of reinforcement learning}.
\newblock PhD thesis, University of London, 2003.

\bibitem[Oh et~al.(2015)Oh, Guo, Lee, Lewis, and Singh]{oh2015action}
Oh, J., Guo, X., Lee, H., Lewis, R.~L., and Singh, S.
\newblock Action-conditional video prediction using deep networks in atari
  games.
\newblock In \emph{Advances in Neural Information Processing Systems 28
  (NIPS)}, pp.\  2845--2853, 2015.

\bibitem[Oh et~al.(2017)Oh, Singh, and Lee]{oh2017value}
Oh, J., Singh, S., and Lee, H.
\newblock Value prediction network.
\newblock In \emph{Advances in Neural Information Processing Systems 30}, pp.\
  6120--6130, 2017.

\bibitem[Ross \& Bagnell(2012)Ross and Bagnell]{ross2012agnostic}
Ross, S. and Bagnell, D.
\newblock Agnostic system identification for model-based reinforcement
  learning.
\newblock In \emph{Proceedings of the 29th International Conference on Machine
  Learning (ICML)}, pp.\  1703--1710, 2012.

\bibitem[Silver et~al.(2017)Silver, van Hasselt, Hessel, Schaul, Guez, Harley,
  Dulac{-}Arnold, Reichert, Rabinowitz, Barreto, and
  Degris]{silver2017predictron}
Silver, D., van Hasselt, H., Hessel, M., Schaul, T., Guez, A., Harley, T.,
  Dulac{-}Arnold, G., Reichert, D.~P., Rabinowitz, N., Barreto, A., and Degris,
  T.
\newblock The predictron: End-to-end learning and planning.
\newblock In \emph{Proceedings of the 34th International Conference on Machine
  Learning (ICML)}, pp.\  3191--3199, 2017.

\bibitem[Sorg et~al.(2010{\natexlab{a}})Sorg, Lewis, and Singh]{sorg2010reward}
Sorg, J., Lewis, R.~L., and Singh, S.
\newblock Reward design via online gradient ascent.
\newblock In \emph{Advances in Neural Information Processing Systems 23
  (NIPS)}, pp.\  2190--2198, 2010{\natexlab{a}}.

\bibitem[Sorg et~al.(2010{\natexlab{b}})Sorg, Singh, and
  Lewis]{sorg2010internal}
Sorg, J., Singh, S.~P., and Lewis, R.~L.
\newblock Internal rewards mitigate agent boundedness.
\newblock In \emph{Proceedings of the 27th International Conference on Machine
  Learning (ICML)}, pp.\  1007--1014, 2010{\natexlab{b}}.

\bibitem[Sorg et~al.(2011)Sorg, Singh, and Lewis]{sorg2011optimal}
Sorg, J., Singh, S.~P., and Lewis, R.~L.
\newblock Optimal rewards versus leaf-evaluation heuristics in planning agents.
\newblock In \emph{Proceedings of the Twenty-Fifth AAAI Conference on
  Artificial Intelligence (AAAI)}, pp.\  465--470, 2011.

\bibitem[Szita \& Szepesv{\'a}ri(2010)Szita and Szepesv{\'a}ri]{szita2010model}
Szita, I. and Szepesv{\'a}ri, C.
\newblock Model-based reinforcement learning with nearly tight exploration
  complexity bounds.
\newblock In \emph{Proceedings of the 27th International Conference on Machine
  Learning (ICML)}, pp.\  1031--1038, 2010.

\bibitem[Talvitie(2014)]{talvitie2014model}
Talvitie, E.
\newblock Model regularization for stable sample rollouts.
\newblock In \emph{Proceedings of the 30th Conference on Uncertainty in
  Artificial Intelligence (UAI)}, pp.\  780--789, 2014.

\bibitem[Talvitie(2015)]{talvitie2015agnostic}
Talvitie, E.
\newblock Agnostic system identification for monte carlo planning.
\newblock In \emph{Proceedings of the 29th AAAI Conference on Artificial
  Intelligence (AAAI)}, pp.\  2986--2992, 2015.

\bibitem[Talvitie(2017)]{talvitie2017self}
Talvitie, E.
\newblock Self-correcting models for model-based reinforcement learning.
\newblock In \emph{Proceedings of the Thirty-First {AAAI} Conference on
  Artificial Intelligence (AAAI)}, pp.\  2597--2603, 2017.

\bibitem[Tamar et~al.(2016)Tamar, Wu, Thomas, Levine, and
  Abbeel]{tamar2016value}
Tamar, A., Wu, Y., Thomas, G., Levine, S., and Abbeel, P.
\newblock Value iteration networks.
\newblock In \emph{Advances in Neural Information Processing Systems 29
  (NIPS)}, pp.\  2154--2162, 2016.

\bibitem[Tesauro \& Galperin(1996)Tesauro and Galperin]{tesauro1996line}
Tesauro, G. and Galperin, G.~R.
\newblock On-line policy improvement using monte-carlo search.
\newblock In \emph{Advances in Neural Information Processing Systems 9 (NIPS)},
  pp.\  1068--1074, 1996.

\bibitem[Veness et~al.(2011)Veness, Ng, Hutter, Uther, and
  Silver]{veness11montecarlo}
Veness, J., Ng, K.~S., Hutter, M., Uther, W. T.~B., and Silver, D.
\newblock {A Monte-Carlo AIXI Approximation}.
\newblock \emph{Journal of Artificial Intelligence Research (JAIR)},
  40:\penalty0 95--142, 2011.

\bibitem[Veness et~al.(2012)Veness, Ng, Hutter, and Bowling]{veness2012context}
Veness, J., Ng, K.~S., Hutter, M., and Bowling, M.
\newblock Context tree switching.
\newblock In \emph{Proceedings of the 2012 Data Compression Conference (DCC)},
  pp.\  327--336, 2012.

\bibitem[Venkatraman et~al.(2015)Venkatraman, Hebert, and
  Bagnell]{venkatraman2015improving}
Venkatraman, A., Hebert, M., and Bagnell, J.~A.
\newblock Improving multi-step prediction of learned time series models.
\newblock In \emph{Proceedings of the 29th AAAI Conference on Artificial
  Intelligence (AAAI)}, pp.\  3024--3030, 2015.

\bibitem[Venkatraman et~al.(2016)Venkatraman, Capobianco, Pinto, Hebert, Nardi,
  and Bagnell]{venkatraman2016improved}
Venkatraman, A., Capobianco, R., Pinto, L., Hebert, M., Nardi, D., and Bagnell,
  J.~A.
\newblock Improved learning of dynamics models for control.
\newblock In \emph{2016 International Symposium on Experimental Robotics}, pp.\
   703--713. Springer, 2016.

\end{thebibliography}
\bibliographystyle{icml2018}

\appendix

\section{Hallucinated DAgger-MC Details}

Hallucinated DAgger-MC, like earlier variations on DAgger, requires
the ability to reset to the initial state distribution $\mu$ and also
the ability to reset to an ``exploration distribution'' $\nu$. The
exploration distribution ideally ensures that the agent will encounter
states that would be visited by a good policy. The performance bound
for H-DAgger-MC depends in part on the quality of the selected $\nu$.

In addition to assuming a particular form for the planner (one-ply MC
with a blind rollout policy), H-DAgger-MC requires the dynamics model
to be ``unrolled''. Rather than learning a single $\hat{P}$,
H-DAgger-MC learns a set
$\{\hat{P}^1, \ldots, \hat{P}^{T-1}\} \subseteq \mathcal{P}$, where
model $\hat{P}^i$ is responsible for predicting the outcome of step
$i$ of a rollout, given the state sampled from $\hat{P}^{i-1}$. While
this impractical condition is important theoretically,
\citet{talvitie2017self} showed that in practice a single
$\mathcal{P}$ can be used for all steps; the experiments in Section
\ref{sec:experiments} make use of this practical alteration.

Algorithm \ref{alg:HDAggerMC} augments H-DAgger-MC to learn a reward
model as well as a dynamics model. In particular, H-DAgger-MC proceeds
in iterations, each iteration producing a new plan, which is turn used
to collect data to train a new model. In each iteration state-action
pairs are sampled using the current plan and the exploration
distribution (lines 7-13), and then the world and model are rolled out
in parallel to generate hallucinated training examples (lines
14-21). The resulting data is used to update the model. We simply add
a reward model learning process, and collect training examples along
with the state transition examples during the rollout. After both parts of
the model have been updated, a new plan is generated for the
subsequent iteration. Note that while the dynamics model is
``unrolled'', there is only a single reward model that is responsible
for predicting the reward at every step of the rollout. We assume that
the reward learning algorithm is performing a weighted regression
(where each training example is weighted by $\gamma^{t-1}$ for the
rollout step $t$ in which it occurred).

\begin{algorithm}[!t]
  \caption{Hallucinated DAgger-MC (+ reward learning)}\label{alg:HDAggerMC}
  \begin{algorithmic}[1]
    \REQUIRE \textsc{Learn-Dynamics}{}, \textsc{Learn-Reward}, exploration distr. $\nu$,
    \textsc{MC-Planner}(blind rollout policy $\rho$, depth
    $T$), \# iterations $N$, \# rollouts per iteration $K$. 
    \STATE Get initial datasets $\mathcal{D}^{1:T-1}_1$ and $\mathcal{E}_1$ (maybe using $\nu$)
    \STATE Initialize $\hat{P}^{1:T-1}_1 \gets$ \textsc{Learn-Dynamics}($\mathcal{D}^{1:T-1}_1$).
    \STATE Initialize $\hat{R}_1 \gets$ \textsc{Learn-Reward}($\mathcal{E}_1$).
    \STATE Initialize $\hat{\pi}_1 \gets$ \textsc{MC-Planner}($\hat{P}^{1:T-1}_1$, $\hat{R}_1$).
    \FOR{$n \gets 2 \ldots N$}
    \FOR{$k \gets 1 \ldots K$}
    \STATE With probability... \COMMENT{First sample from $\xi$}\label{line:samplestart}
    \STATE \hspace{1em} $\nicefrac{1}{2}$: Sample $(x, b) \sim D_{\mu}^{\hat{\pi}_n}$
    \STATE \hspace{1em} $\nicefrac{1}{4}$: Reset to $(x, b) \sim \nu$.
    \STATE \hspace{1em} $\nicefrac{(1-\gamma)}{4}$: Sample $x \sim \mu$, $b \sim
     \hat{\pi}_n(\cdot \mid x)$. 
     \STATE \hspace{1em} $\nicefrac{\gamma}{4}$: Reset to $(y, c)
     \sim \nu$
     \STATE \hspace{1em}\hspace{1em}\hspace{0.1in}
     Sample $x \sim P(\cdot \mid y, c)$, $b \sim \hat{\pi}_n(\cdot \mid x)$\label{line:sampleend}
    \STATE Let $s \gets x$, $z \gets x$, $a \gets b$.
    \FOR[Parallel rollouts...]{$t \leftarrow 1\ldots T-1$} \label{line:startproll}
    \STATE Sample $s' \sim P(\cdot \mid s, a)$.
    \STATE Add $\langle z, a, s' \rangle$ to
    $\mathcal{D}^t_n$. \\
    \hfill \COMMENT{(DAgger-MC adds $\langle s, a, s' \rangle$)}
    \STATE Add $\langle z, a, R_s^a, \gamma^{t-1} \rangle$ to
    $\mathcal{E}_n$.\\
    \hfill \COMMENT{(Standard approach adds $\langle s,
    a, R_s^a, \gamma^{t-1} \rangle$)} 
    \STATE Sample $z' \sim \hat{P}^t_{n-1}(\cdot \mid z, a)$.
    \STATE Let $s \gets s'$, $z \gets z'$, and sample $a \sim \rho$.\label{line:endproll}
    \ENDFOR
    \STATE Add $\langle z, a, R_s^a, \gamma^{T-1} \rangle$ to
    $\mathcal{E}_n$.\\
    \hfill \COMMENT{(Standard approach adds $\langle s,
    a, R_s^a, \gamma^{T-1} \rangle$)} 
    \ENDFOR
    \STATE $\hat{P}^{1:T-1}_n \gets$ \textsc{Learn-Dynamics}($\hat{P}^{1:T-1}_{n-1}$, $\mathcal{D}^{1:T-1}_n$) 
    \STATE $\hat{R}_n \gets$ \textsc{Learn-Reward}($\hat{R}_{n-1}$, $\mathcal{E}_n$) 
    \STATE $\hat{\pi}_n \gets$ \textsc{MC-Planner}($\hat{P}^{1:T-1}_{n}$, $\hat{R}_n$).
    \ENDFOR
    \STATE \textbf{return} the sequence $\hat{\pi}_{1:N}$
  \end{algorithmic}
\end{algorithm}

\subsection{Analysis of H-DAgger-MC}

We now derive theoretical guarantees for this new version of
H-DAgger-MC. The analysis is similar to that of existing DAgger
variants \cite{ross2012agnostic, talvitie2015agnostic,
  talvitie2017self}, but the proof is included for completeness. Let
$H^t_n$ be the distribution from which H-DAgger-MC samples a training
example at depth $t$ (lines 7-13 to pick an initial state-action pair,
lines 14-21 to roll out). Define the average error of the dynamics
model at depth $t$ to be
\begin{align*}
\bar{\epsilon}^{t}_{prd} = \frac{1}{N}\sum_{n = 1}^N \expects_{(s, z,
  a) \sim H^t_n} [1 -\hat{P}_{n}^{t}(\sigma_{s}^{a} \mid z, a)].
\end{align*}
Let $\epsilon_{\hat{R}_n}(s, z, a) = |R(s, a) - \hat{R}_n(z,
a)|$ and let
\begin{align*}
\bar{\epsilon}_{hrwd} = \frac{1}{N}\sum_{n = 1}^N \sum_{t=1}^T
\gamma^{t-1} \expects_{(s, z, a) \sim H^t_n} [\epsilon_{\hat{R}_n}(s,
  z, a)|]
\end{align*}
be the average reward model error. Finally, let $D^t_n$ be the
distribution from which H-DAgger-MC samples $s$ and $a$ during the
rollout in lines 14-21. The error of the reward model with respect to
these environment states is
\begin{align*}
\bar{\epsilon}_{erwd} = \frac{1}{N}\sum_{n=1}^N \sum_{t=1}^{T}
\gamma^{t-1} \expects_{(s, a) \sim D^t_n} [|R(s, a) -
\hat{R}(s, a)|].
\end{align*}

For a policy $\pi$, let
$c_{\nu}^{\pi} = \sup_{s,a} \frac{D_{\mu, \pi}(s,a)}{\nu(s,a)}$
represent the mismatch between the discounted state-action
distribution under $\pi$ and the exploration distribution $\nu$. Now,
consider the sequence of policies $\hat{\pi}_{1:N}$ generated by
H-DAgger-MC. Let $\bar{\pi}$ be the uniform mixture over all policies
in the sequence. Let
$\bar{\epsilon}_{mc} = \frac{1}{N} \frac{4}{1 - \gamma} \sum_{n=1}^N
\|\bar{Q}_n - \hat{Q}^{\rho}_{T,n}\|_{\infty} + \frac{2}{1 - \gamma}
\|BV^{\rho}_T - V^{\rho}_T\|_{\infty} $
be the error induced by the choice of planning algorithm, averaged
over all iterations.
\begin{lemma} \label{lem:prederror}
  In H-DAgger-MC, the policies $\hat{\pi}_{1:N}$ are such that for any
  policy $\pi$,
\begin{align*}
 \expect_{s
  \sim \mu}&\big[V^{\pi}(s) - V^{\bar{\pi}}(s)\big] \le
  \frac{4}{1-\gamma}c^\pi_\nu \bar{\epsilon}_{hrwd} + \bar{\epsilon}_{mc}\\
&\le \frac{4}{1-\gamma} c^\pi_\nu \Big(\bar{\epsilon}_{erwd} + 2M
  \sum_{t=1}^{T-1}\gamma^{t-1} \bar{\epsilon}^t_{prd}\Big) + \bar{\epsilon}_{mc}.
\end{align*}
\end{lemma}
\begin{proof}
Recall that 
\begin{align*}
&\expect_{s \sim \mu}\big[V^{\pi}(s) - V^{\bar{\pi}}(s)\big] = \frac{1}{N}\sum_{n = 1}^N \expect_{s \sim \mu} \big[V^{\pi}(s) -
V^{\hat{\pi}_n}(s)\big].&
\end{align*}
and by Lemma \ref{lem:mcvaluebound} for any $n \ge 1$, 
\begin{align*}
&\expect_{s \sim \mu} \big[V^{\pi}(s) -
V^{\hat{\pi}_n}(s)\big] \le&\\
&\hspace{.4in}\frac{4}{1 - \gamma} \expect_{(s, a)
  \sim \xi^{\pi, \hat{\pi}_n}_{\mu}}[|\hat{Q}_{T,n}^{\rho}(s, a) - Q_T^{\rho}(s, a)|] +\bar{\epsilon}_{mc},&
\end{align*}
where 
\begin{align*}
\xi^{\pi, \hat{\pi}_n}_{\mu}(s, a) = \frac{1}{2} &D_{\mu, \hat{\pi}_n}(s, a) + \frac{1}{4}
  D_{\mu, \pi}(s, a)\\ 
&+ \frac{1}{4}\Big((1 - \gamma) \mu(s)
    \hat{\pi}_n(a \mid s)\\  
&\hspace{.1in} + \gamma \sum_{z, b} D_{\mu, \pi}(z, b)
    P_{z}^{b}(s) \hat{\pi}_n(a \mid s) \Big).
\end{align*} 

Then, combining the above with Theorem \ref{thm:hrwd},
\begin{align*}
&\frac{1}{N} \sum_{n=1}^N \frac{4}{1 - \gamma} \expect_{(s, a)
  \sim \xi^{\pi, \hat{\pi}_n}_{\mu}}[|\hat{Q}_{T,n}^{\rho}(s, a) - Q_T^{\rho}(s, a)|] +\bar{\epsilon}_{mc}&\\
& \le \frac{1}{N} \sum_{n=1}^N \frac{4}{1-\gamma}
  \sum_{t=1}^T \gamma^{t-1} \expect_{\scriptsize\begin{aligned}&(s, z, a)\\
  &\sim
  H^{t,n}_{\xi_\mu^{\pi, \hat{\pi}_n}, \rho}\end{aligned}} [\epsilon_{\hat{R}_n}(s,
  z, a)] + \bar{\epsilon}_{mc}&
\end{align*}

Now note that for any $t$ and any $n$,
\begin{align*}
&\expect_{(s, z, a) \sim H^{t, n}_{\xi^{\pi,
      \hat{\pi}_n}_{\mu}, \rho}}\big[\epsilon_{\hat{R}_n}(s, z,
  a)\big]&\\
&= \frac{1}{2} \sum_{s', a'} D_{\mu, \hat{\pi}_n}(s', a') \expect_{(s, z, a) \sim H^{t, n}_{s', a', \rho}}\big[\epsilon_{\hat{R}_n}(s, z, a)\big]&\\
&\hspace{.15in}+\frac{1}{4} \sum_{s', a'} D_{\mu, \pi}(s', a') \expect_{(s, z, a) \sim H^{t, n}_{s', a', \rho}}\big[\epsilon_{\hat{R}_n}(s, z, a)\big]&\\
&\hspace{.15in}+\frac{\gamma}{4} \sum_{s', a'} \sum_{s'', a''}D_{\mu, \pi}(s'',
a'')P_{s''}^{a''}(s')\hat{\pi}_n(a' \mid s')&\\
&\hspace{1.3in} \expect_{(s, z, a) \sim H^{t, n}_{s', a', \rho}}\big[\epsilon_{\hat{R}_n}(s, z, a)\big]&\\
&\hspace{.15in}+\frac{1- \gamma}{4} \sum_{s', a'}
\mu(s')\hat{\pi}_n(a' \mid s')&\\
&\hspace{1.3in} \expect_{(s, z, a) \sim H^{t, n}_{s', a',
  \rho}}\big[\epsilon_{\hat{R}_n}(s, z, a)\big]&\\
&\le \frac{1}{2} \sum_{s', a'} D_{\mu, \hat{\pi}_n}(s', a') \expect_{(s, z, a) \sim H^{t, n}_{s', a', \rho}}\big[\epsilon_{\hat{R}_n}(s, z, a)\big]&\\
&\hspace{.15in}+\frac{1}{4} c^\pi_\nu \sum_{s', a'} \nu(s', a') \expect_{(s, z, a) \sim H^{t, n}_{s', a', \rho}}\big[\epsilon_{\hat{R}_n}(s, z, a)\big]&\\
&\hspace{.15in}+\frac{\gamma}{4} c^\pi_\nu \sum_{s', a'} \sum_{s'', a''}\nu(s'',
a'')P_{s''}^{a''}(s')\hat{\pi}_n(a' \mid s')&\\
&\hspace{1.3in} \expect_{(s, z, a) \sim H^{t, n}_{s', a', \rho}}\big[\epsilon_{\hat{R}_n}(s, z, a)\big]&\\
&\hspace{.15in}+\frac{1- \gamma}{4} \sum_{s', a'}
\mu(s')\hat{\pi}_n(a' \mid s')&\\
&\hspace{1.3in} \expect_{(s, z, a) \sim H^{t, n}_{s', a',
  \rho}}\big[\epsilon_{\hat{R}_n}(s, z, a)\big]&
\end{align*}
\begin{align*}
&\le c^\pi_\nu \bigg(\frac{1}{2} \sum_{s', a'} D_{\mu,
  \hat{\pi}_n}(s', a')&\\
&\hspace{1.3in} \expect_{(s, z, a) \sim H^{t, n}_{s', a', \rho}}\big[\epsilon_{\hat{R}_n}(s, z, a)\big]&\\
&\hspace{.15in}+\frac{1}{4} \sum_{s', a'} \nu(s', a') \expect_{(s, z, a) \sim H^{t, n}_{s', a', \rho}}\big[\epsilon_{\hat{R}_n}(s, z, a)\big]&\\
&\hspace{.15in}+\frac{\gamma}{4} \sum_{s', a'} \sum_{s'', a''}\nu(s'',
a'')P_{s''}^{a''}(s')\hat{\pi}_n(a' \mid s')&\\
&\hspace{1.3in} \expect_{(s, z, a) \sim H^{t, n}_{s', a', \rho}}\big[\epsilon_{\hat{R}_n}(s, z, a)\big]&\\
&\hspace{.15in}+\frac{1- \gamma}{4} \sum_{s', a'} \mu(s')\hat{\pi}_n(a'
\mid s')&\\
&\hspace{1.3in} \expect_{(s, z, a) \sim H^{t, n}_{s', a',
  \rho}}\big[\epsilon_{\hat{R}_n}(s, z, a)\big]&\\
&=c^\pi_\nu \expect_{(s, z, a) \sim H^{t, n}_{\xi_n, \rho}}\big[\epsilon_{\hat{R}_n}(s, z, a)\big].&
\end{align*}

When $t = 1$,
\begin{align*}
\expect_{(s, z, a) \sim H^{t, n}_{\xi_n,
    \rho}}\big[\epsilon_{\hat{R}_n}(s, z, a)\big] = \expect_{(s, a) \sim \xi_n(s,
  a)}\big[\epsilon_{\hat{R}_n}(s, z, a)\big]. 
\end{align*}
When $t > 1$,
\begin{align*}
&\expect_{(s, z, a) \sim H^{t, n}_{\xi_n, \rho}}\big[\epsilon_{\hat{R}_n}(s, z, a)\big]&\\
&= \sum_{s_t, z_t, a_t}\expect_{(s_1, a_1) \sim
  \xi_n}\bigg[\sum_{a_{1:t-1}}\rho(a_{2:t} \mid a_1)&\\
&\hspace{.5in}P_{s_1}^{a_{0:t-1}}(s_t \mid
s_1, a_{0:t-1})\hat{P}_{n}^{1:t-1}(z_t \mid
s_1,a_{0:t-1})&\\
&\hspace{2.4in}\epsilon_{\hat{R}_n}(s_t, z_t, a_t)\bigg]&\\
&=\expect_{(s, z, a) \sim H^t_n}\big[\epsilon_{\hat{R}_n}(s, z, a)\big].&
\end{align*}

Thus, putting it all together, we have shown that
\begin{align*}
  &\expect_{s \sim \mu}\big[V^{\pi}(s) - V^{\bar{\pi}}(s)\big] &\\
&\hspace{.1in}\le
  \frac{4}{1 - \gamma} c^\pi_{\nu}\frac{1}{N} \sum_{n=1}^N
  \sum_{t = 1}^{T}\gamma^{t-1}\expect_{(s, z, a) \sim H^t_n}\big[\epsilon_{\hat{R}_n}(s, z,
  a)\big]&\\
&\hspace{2.8in} + \bar{\epsilon}_{mc}&\\
&\hspace{.1in}=
  \frac{4}{1 - \gamma} c^\pi_{\nu}\bar{\epsilon}_{hrwd} + \bar{\epsilon}_{mc}.&
\end{align*}
Thus we have proven the first inequality. Furthermore, by Theorem \ref{thm:tightness},
\begin{align*}
&\bar{\epsilon}_{hrwd} = \frac{1}{N} \sum_{n=1}^N
  \sum_{t = 1}^{T}\expect_{(s, z, a) \sim H^t_n}\big[\hat{R}_n(s, z,
  a)\big]&\\
&\le \frac{1}{N} \sum_{n=1}^N
  \bigg(\sum_{t = 1}^{T}\gamma^{t-1}\expect_{(s, a) \sim D^t_n}\big[|R(s, a) -
  \hat{R}_n(s, a)|\big]&\\
&\hspace{.45in} + 2M\sum_{t = 1}^{T-1}\gamma^{t-1}\expect_{(s, z, a) \sim H^t_n}\big[1 -
  \hat{P}^t_n(\sigma_s^a \mid z, a)|\big]\bigg)&\\
&= \frac{1}{N} \sum_{n=1}^N\sum_{t = 1}^{T}\gamma^{t-1}\expect_{(s, a) \sim D^t_n}\big[|R(s, a) -
  \hat{R}_n(s, a)|\big]&\\
&\hspace{.2in} + 2M\sum_{t = 1}^{T-1}\gamma^{t-1}\frac{1}{N} \sum_{n=1}^N\expect_{(s, z, a) \sim H^t_n}\big[1 -
  \hat{P}^t_n(\sigma_s^a \mid z, a)|\big]&\\
&\le \bar{\epsilon}_{erwd} + 2M
  \sum_{t=1}^{T-1}\gamma^{t-1} \bar{\epsilon}^t_{prd}.
\end{align*}
This gives the second inequality.
\end{proof}

Note that this result holds for {\em any} comparison policy
$\pi$. Thus, if $\bar{\epsilon}_{mc}$ is small and the learned models
have low error, then if $\nu$ is similar to the state-action
distribution under {\em some} good policy, $\bar{\pi}$ will compare
favorably to it. That said, Lemma \ref{lem:prederror} shares the
limitations of the comparable results for the other DAgger
algorithms. It focuses on the L1 loss, which is not always a practical
learning objective. It also assumes that the expected loss at each
iteration can be computed exactly (i.e. that there are infinitely many
samples per iteration). It also applies to the average policy
$\bar{\pi}$, rather than $\hat{\pi}_N$. \citet{ross2012agnostic}
discuss extensions that address more practical loss functions, finite
sample bounds, and results for $\hat{\pi}_N$.

Lemma \ref{lem:prederror} effectively says that {\em if} the models
have low training error, the resulting policy will be good. It does
not promise that the models will have low training error. Following
\citet{ross2012agnostic} note that $\bar{\epsilon}^t_{prd}$ and
$\bar{\epsilon}_{hrwd}$ can each be interpreted as the average loss of
an online learner on the problem defined by the aggregated
datasets. Then for each horizon depth $t$ let
$\bar{\epsilon}^{t}_{pmdl}$ be the error of the best dynamics model in
$\mathcal{P}$ under the training distribution at that depth, in
retrospect. Specifically,
\begin{align*}
  \bar{\epsilon}^{t}_{\mathcal{P}} = \inf_{P' \in
  \mathcal{P}}\frac{1}{N}\sum_{n = 1}^N \expect_{(s, z, a) \sim
  H^t_n} [1 - P'(\sigma_{s}^{a} \mid z, a)].
\end{align*}
Similarly, let
\begin{align*}
\bar{\epsilon}_{\mathcal{R}} = \inf_{R' \in \mathcal{R}}\frac{1}{N}\sum_{n =
  1}^N \sum_{t=1}^T \gamma^{t-1} \expect_{(s, z, a) \sim
  \tilde{H}^t_n} [\epsilon_{R'}(s, z, a)]
\end{align*}
 be the error of the best reward
model in $\mathcal{R}$ in retrospect.

The average regret for the dynamics model at depth $t$ is
$\bar{\epsilon}^t_{prgt} = \bar{\epsilon}_{prd}^t -
\bar{\epsilon}_{\mathcal{P}}^t$.
For the reward model it is
$\bar{\epsilon}_{rrgt} = \bar{\epsilon}_{hrwd} -
\bar{\epsilon}_{\mathcal{R}}$.
For a no-regret online learning algorithm, average regret approaches 0
as $N \rightarrow \infty$. This gives the following bound on
H-DAgger-MC's performance in terms of model regret.
\begin{theorem}\label{thm:hdaggermc}
  In H-DAgger-MC, the policies $\hat{\pi}_{1:N}$ are such that for any
  policy $\pi$,
\begin{align*}
& \expect_{s
  \sim \mu}\big[V^{\pi}(s) - V^{\bar{\pi}}(s)\big]\\
& \le
  \frac{4}{1-\gamma}c^\pi_\nu (\bar{\epsilon}_{\mathcal{R}} + \bar{\epsilon}_{rrgt}) + \bar{\epsilon}_{mc}\\
& \le \frac{4}{1-\gamma} c^\pi_\nu \Big(\bar{\epsilon}_{erwd} + 2M
  \sum_{t=1}^{T-1} \gamma^{t-1} (\bar{\epsilon}^t_{\mathcal{P}} + \bar{\epsilon}^t_{prgt})\Big) + \bar{\epsilon}_{mc}
\end{align*}
  and if the learning algorithms are no-regret then as $N \rightarrow \infty$,
  $\bar{\epsilon}_{rrgt} \rightarrow 0$ and $\bar{\epsilon}_{prgt}^t \rightarrow 0$ for
  each $1 \le t \le T-1$.
\end{theorem}

Theorem \ref{thm:hdaggermc} says that if $\mathcal{R}$ contains a
low-error reward model relative to the learned dynamics models then,
as discussed above, if $\bar{\epsilon}_{mc}$ is small and $\nu$ visits
important states, the resulting policy will yield good performance.
If $\mathcal{P}$ and $\mathcal{R}$ contain perfect models, $\bar{\pi}$
will be comparable to the plan generated by the perfect model.

As noted by \citet{talvitie2017self}, this result does {\em not}
promise that H-DAgger-MC will eventually achieve the performance of
the best available set of dynamics models. The model at each rollout
depth is trained to minimize prediction error given the input
distribution provided by the shallower models without regard for the
effect on deeper models. It is possible that better overall error
could be achieved by {\em increasing} the prediction error at one
depth in exchange for a favorable state distribution for deeper
models. Similarly, as discussed in Section \ref{sec:implications},
H-DAgger-MC will not necessarily achieve the performance of the best
available combination of dynamics and reward models. The dynamics
model is trained without regard for the impact on the reward model. It
could be that a dynamics model with higher prediction error would
allow for lower hallucinated reward error. H-DAgger-MC does not take
this possibility into account.
\end{document}